\documentclass{article}
\usepackage[latin9]{inputenc}
\usepackage{prettyref}
\usepackage{amsthm}
\usepackage{amsmath}
\usepackage[authoryear]{natbib}

\makeatletter
\usepackage{enumitem}		
  \theoremstyle{definition}
  \newtheorem{defn}{\protect\definitionname}
  \theoremstyle{plain}
  \newtheorem{lem}{\protect\lemmaname}
 \theoremstyle{definition}
  \newtheorem{example}{\protect\examplename}
  \theoremstyle{definition}
  \newtheorem*{example*}{\protect\examplename}
  \theoremstyle{remark}
  \newtheorem{claim}{\protect\claimname}
  \theoremstyle{plain}
  \newtheorem{cor}{\protect\corollaryname}
  \theoremstyle{plain}
  \newtheorem*{lem*}{\protect\lemmaname}
  \theoremstyle{remark}
  \newtheorem*{claim*}{\protect\claimname}

\usepackage{times}
\usepackage{graphicx} 
\usepackage{subfigure} 

\usepackage{natbib}

\usepackage{algorithmicx}
\usepackage[noend]{algpseudocode}

\usepackage{proceed2e}

\title{Lifted Representation of Relational Causal Models Revisited: \\ Implications for Reasoning and Structure Learning}

\author{ {\bf Sanghack Lee \and Vasant Honavar} \\
Artificial Intelligence Research Laboratory \\
College of Information Sciences and Technology \\
Pennsylvania State University\\
University Park, PA 16802 \\
}

\usepackage{amssymb}
\usepackage{helvet}
\usepackage{courier}
\newenvironment{proof-sketch}{\noindent{\emph{Proof sketch.}}\hspace*{1em}}{\qed\bigskip\\}

\usepackage{prettyref}
\newrefformat{sec}{Section~\ref{#1}}
\newrefformat{clm}{Claim~\ref{#1}}
\newrefformat{lem}{Lemma~\ref{#1}}
\newrefformat{cor}{Corollary~\ref{#1}}
\newrefformat{prop}{Proposition~\ref{#1}}
\newrefformat{subs}{Section~\ref{#1}}
\newrefformat{sub}{Section~\ref{#1}}
\newrefformat{alg}{Algorithm~\ref{#1}}
\newrefformat{fig}{Figure~\ref{#1}}
\newrefformat{def}{Definition~\ref{#1}}
\newrefformat{dfn}{Definition~\ref{#1}}
\newrefformat{eq}{Eq.~\ref{#1}}
\newrefformat{fact}{Fact~\ref{#1}}
\newrefformat{thm}{Theorem~\ref{#1}}
\newrefformat{fn}{Footnote~\ref{#1}}

\newcommand\Perp{\protect\mathpalette{\protect\independenT}{\perp}} \def\independenT#1#2{\mathrel{\rlap{$#1#2$}\mkern3mu{#1#2}}}

\usepackage{pgf}
\usepackage{tikz}
\usetikzlibrary{arrows,automata}
\usetikzlibrary{decorations,decorations.pathmorphing,decorations.pathreplacing}
\usetikzlibrary{shapes,calc}
\usetikzlibrary{shapes.multipart}
\usetikzlibrary{matrix}
\usetikzlibrary{positioning}
\usetikzlibrary{shadows}
\usetikzlibrary{calc}

\usepackage{stmaryrd}

\algnewcommand{\IIf}[1]{\State\algorithmicfor\  #1\ \algorithmicdo\ }
\algnewcommand{\EndIIf}{\unskip\ }

\makeatother

  \providecommand{\claimname}{Claim}
  \providecommand{\definitionname}{Definition}
  \providecommand{\examplename}{Example}
  \providecommand{\lemmaname}{Lemma}
\providecommand{\corollaryname}{Corollary}

\begin{document}
\maketitle
\begin{abstract}
\citet{maier2010rpc} introduced\emph{ the relational causal model}
(RCM) for representing and inferring causal relationships in relational
data. A lifted representation, called \emph{abstract ground graph}
(AGG), plays a central role in reasoning with and learning of RCM.
The correctness of the algorithm proposed by \citet{maier2013rcd}
for learning RCM from data relies on the \emph{soundness and completeness}
of AGG for \emph{relational d-separation} to reduce the learning of
an RCM to learning of an AGG. We revisit the definition of AGG and
show that AGG, as defined in \citet{maier2013rds-workshop}, does
\emph{not} correctly abstract all ground graphs. We revise the definition
of AGG to ensure that it correctly abstracts all ground graphs. We
further show that AGG representation is \emph{not complete} for relational
\emph{d}-separation\emph{,} that is, there can exist conditional independence
relations in an RCM that are not entailed by AGG. A careful examination
of the relationship between the lack of completeness of AGG for relational
\emph{d}-separation and\emph{ faithfulness} conditions suggests that
weaker notions of completeness, namely \emph{adjacency faithfulness
}and \emph{orientation faithfulness} between an RCM and its AGG, can
be used to learn an RCM from data.
\end{abstract}

\section{INTRODUCTION}

Discovery of causal relationships from observational and experimental
data is a central problem with applications across multiple areas
of scientific endeavor. There has been considerable progress over
the past decades on algorithms for eliciting causal relationships
from data under a broad range of assumptions \citep{pearl2000causality,spirtes2000causation,shimizu2006linear}.
Most algorithms for causal discovery assume propositional data where
instances are independent and identically distributed. However, in
many real world applications, these assumptions are violated because
the underlying data has a relational structure of the sort that is
modeled in practice by an entity-relationship model \citep{chen1976entity}.
There has been considerable work on learning predictive models from
relational data \citep{getoor2007introduction}. Furthermore, researchers
from different disciplines have studied causal relationships and resulting
phenomena on relational world, e.g., peer effects \citep{sacerdote2000peer,ogburn2014interference},
social contagion \citep{christakis2007spread,shalizi2011}, viral
marketing \citep{leskovec2007dynamics}, and information diffusion
\citep{gruhl2004information}.

Motivated by the limitations of traditional approaches to learning
causal relationships from relational data, Maier and his colleagues
introduced the relational causal model (RCM) \citep{maier2010rpc}
and provided a sound and complete causal structure learning algorithm,
called the relational causal discovery (RCD) algorithm \citep{maier2013rcd},
for inferring causal relationships from relational data. The key idea
behind RCM is that a cause and its effects are in a direct or indirect
relationship that is reflected in the relational data. Traditional
approaches for reasoning on and learning of a causal model cannot
be trivially applied for relational causal model \citep{maier2013rcd}.
Reasoning on an RCM to infer a relational version of conditional independence
(CI) makes use of a lifted representation, called \emph{abstract ground
graphs} (AGGs), in which traditional graphical criteria can be used
to answer relational CI queries. The lifted representation is employed
as an internal learning structure in RCD to reflect the inferred CI
results among relational version of variables. RCD makes use of a
new orientation rule designed specifically for RCM.

\paragraph{Motivation and Contributions}

RCM \citep{maier2010rpc} offer an attractive model for representing,
reasoning about, and learning causal relationships implicit in relational
data. \citet{arbour2014psm} proposed a relational version of propensity
score matching method to infer (relational) causal effects from observational
data. \citet{marazopoulou2015trcm} extended RCM to cope with \emph{temporal}
relational data. They generalized both RCM and RCD to Temporal RCM
and Temporal RCD, respectively. A lifted representation, called \emph{abstract
ground graph} (AGG), plays a central role in reasoning with and learning
of RCM. The correctness of the algorithms proposed by \citet{maier2013rcd}
for learning RCM and \citet{marazopoulou2015trcm} for Temporal RCM,
respectively, from observational data rely on the \emph{soundness
and completeness} of AGG for \emph{relational d-separation} to reduce
the learning of an RCM to learning of an AGG. The main contributions
of this paper are as follows: (i) We show that AGG, as defined in
\citet{maier2013rds-workshop} does \emph{not} correctly abstract
all ground graphs; (ii) We revise the definition of AGG to ensure
that it correctly abstracts all ground graphs; (iii) We further show
that AGG representation is \emph{not complete} for relational \emph{d}-separation\emph{,}
that is, there can exist conditional independence relations in an
RCM that are not entailed by AGG; and (iv) Based on a careful examination
of the relationship between the lack of completeness of AGG for relational
\emph{d}-separation and\emph{ faithfulness} conditions suggests that
weaker notions of completeness, namely \emph{adjacency faithfulness
}and \emph{orientation faithfulness} between an RCM and its AGG, can
be used to learn an RCM from data.

\section{\label{sec:PRELIMINARIES}PRELIMINARIES}

We follow notational conventions introduced in \citep{maier2013rcd,maier2013rds-workshop,maier2014thesis}.
An entity-relationship model \citep{chen1976entity} abstracts the
\emph{entities }(e.g., \emph{employee}, \emph{product}) and \emph{relationships
}(e.g., \emph{develops}) between entities in a domain using a \emph{relational
schema}. The instantiation of the schema is called a \emph{skeleton}
where entities form a network of relationships (e.g., \emph{Quinn-develops-Laptop},
\emph{Roger-develops-Laptop}). Entities and relationships have attributes
(e.g., \emph{salary} of employees, \emph{success} of products). \emph{Cardinality}
\emph{constraints} specify the cardinality of relationships that an
entity can participate in (e.g., \emph{many} employees \emph{can}
develop a product.).\footnote{The examples are taken from \citet{maier2014thesis}.}
The following definitions are taken from \citet{maier2014thesis}:
\begin{defn}
A \emph{relational schema} $\mathcal{S}$ is a tuple $\left\langle \mathcal{E},\mathcal{R},\mathcal{A},\mathsf{card}\right\rangle $:
a set of entity classes $\mathcal{E}$; a set of relationship classes
$\mathcal{R}$ where $R_{i}=\langle E_{j}^{i}\rangle_{j=1}^{n}$ and
$n=\left|R_{i}\right|$ is arity for $R_{i}$; attribute classes $\mathcal{A}$
where $\mathcal{A}\left(I\right)$ is a set of attribute classes of
$I\in\mathcal{E}\cup\mathcal{R}$; and cardinalities $\mathsf{card}:\mathcal{R}\times\mathcal{E}\!\rightarrow\!\left\{ \mathsf{one},\mathsf{many}\right\} $.
\end{defn}
Every relationship class $R_{i}$ have two or more distinct entity
classes.\footnote{In general, the same entity class can participate in a relationship
class in two or more different roles. For simplicity, we only consider
relationship classes only with distinct entity classes. } We denote by $\mathcal{I}$ all item classes $\mathcal{E}\cup\mathcal{R}$.
We denote by $I_{X}$ an item class that has an attribute class $X$
assuming, without loss of generality, that the attributes of different
item classes are disjoint. Participation of an entity class $E_{j}$
in a relationship class $R_{i}$ is denoted by $E_{j}\in R_{i}$ if
$\exists_{k=1}^{\left|R_{i}\right|}E_{k}^{i}=E_{j}$.

\begin{defn}
\label{dfn:rcm:RS-relational-skeleton}A \emph{relational skeleton}
$\sigma$ is an instantiation of relational schema $\mathcal{S}$,
represented by a graph of entities and relationships. Let $\sigma\left(I\right)$
denote a set of items of item class $I\in\mathcal{I}$ in $\sigma$.
Let $i_{j},i_{k}\in\sigma$ such that $i_{j}\in\sigma(I_{j})$, $i_{k}\in\sigma(I_{k})$,
and $I_{j},I_{k}\in\mathcal{I}$, then we denote $i_{j}\sim i_{k}$
if there exists an edge between $i_{j}$ and $i_{k}$ in $\sigma$.
\end{defn}

\subsection{\label{sub:RELATIONAL-CAUSAL-MODEL}RELATIONAL CAUSAL MODEL}

\emph{Relational causal model} (RCM, \citealp{maier2010rpc}) is a
causal model where causes and their effects are \emph{related} given
an underlying relational schema. For example, the success of a product
depends on the competence of employees who develop the product (see
\prettyref{fig:rcm_example}). 
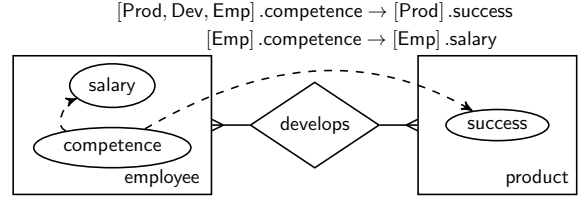
\begin{figure}
\center\begin{tikzpicture}[scale=0.75, every node/.style={scale=0.75},->,>=stealth',auto,semithick]
\tikzset{rel state/.style={draw,diamond,aspect=1.5}}
\tikzset{entity state/.style={draw,rectangle}}
\tikzset{attr state/.style={draw,ellipse}}
\tikzstyle{every state}=[rectangle,fill=white,draw=none,text=black]

\node[entity state,minimum width=10em, minimum height=7em] (A) [label={[xshift=2.5em,yshift=-7em]$\mathsf{employee}$}] {$$};
\node[rel state]  (B) [right=0.5cm of A] {$\mathsf{develops}$};
\node[entity state,minimum width=8em, minimum height=7em] (C) [right=0.5cm of B,label={[xshift=2em,yshift=-7em]$\mathsf{product}$}] {};
\node[attr state] (D) at ($(A)+ (0,0.7)$) {$\mathsf{salary}$};
\node[attr state] (E) at ($(A)+(0,-0.4)$) {$\mathsf{competence}$};
\node[attr state] (F) at ($(C)+(0,0)$) {$\mathsf{success}$};
\node             (G) at ($(B)+(0,2)$) {$\left[\mathsf{Prod},\mathsf{Dev},\mathsf{Emp}\right].\mathsf{competence}\rightarrow\left[\mathsf{Prod}\right].\mathsf{success}$};
\node             (H) at ($(B)+(0.65,1.5)$) {$\left[\mathsf{Emp}\right].\mathsf{competence}\rightarrow\left[\mathsf{Emp}\right].\mathsf{salary}$};

\draw[->,dashed]	(E) to[out=30,in=150]  (F) ;
\draw[->,dashed]	(E) to[out=160,in=200]  (D) ;
\draw[-]   ($(B.east)!0.7!(C.west)$) to ($(C.west)+(0,1mm)$);
\draw[-]   ($(B.east)!0.7!(C.west)$) to ($(C.west)+(0,-1mm)$);
\draw[-]   ($(B.west)!0.7!(A.east)$) to ($(A.east)+(0,1mm)$);
\draw[-]   ($(B.west)!0.7!(A.east)$) to ($(A.east)+(0,-1mm)$);
\draw[-]	(A) to (B);
\draw[-]	(B) to (C);

\end{tikzpicture}\protect\caption{\label{fig:rcm_example}A toy example of RCM adopted from \citet{maier2014thesis}
with two relational dependencies: (i) the success of a product depends
on the competence of employees who develop it; (ii) employee's salary
is affected by his/her competence.}
\end{figure}
 An RCM models \emph{relational dependencies}; each relational dependency
has a cause and its effect, which are represented by \emph{relational
variable}s; a relational variable is a pair consisting of a \emph{relational
path} and an attribute. 
\begin{defn}
\label{dfn:rcm:RP-relational-path}A \emph{relational path} $P=\left[I_{j},\dotsc,I_{k}\right]$
is an alternating sequence of entity class $E\in\mathcal{E}$ and
relationship class $R\in\mathcal{R}$. An item class $I_{j}$ is called
\emph{base} \emph{class} or \emph{perspective} and $I_{k}$ is called
a \emph{terminal} \emph{class}. A relational path should satisfy:
\begin{enumerate}
\item \itemsep0em  for every $\left[E,R\right]$ or $\left[R,E\right]$,
$E\in R$;
\item for every $\left[E,R,E^{\prime}\right]$, $E\neq E^{\prime}$; and
\item for every $\left[R,E,R^{\prime}\right]$, if $R=R^{\prime}$, then
$\mathsf{card}\left(R,E\right)=\mathsf{many}$.
\end{enumerate}
\end{defn}
All valid relational paths on the given schema $\mathcal{S}$ are
denoted by $\mathbf{P}_{\mathcal{S}}$. We denote the \emph{length}
of $P$ by $\left|P\right|$, a \emph{subpath} by $P^{i:j}=\left[P_{k}\right]_{k=i}^{j}$
or $P^{i:}=\left[P_{k}\right]_{k=i}^{\left|P\right|}$ for $1\leq i\leq j\leq\left|P\right|$,
and the \emph{reversed} \emph{path} by $\tilde{P}=[P_{\left|P\right|},\dotsc,P_{2},P_{1}]$.
Note that all subpaths of a relational path as well as the corresponding
reverse paths are valid. A\emph{ relational variable} $P.X$ is a
pair of a relational path $P$ and an attribute class $X$ for the
terminal class of $P$. A relational variable is said to be \emph{canonical}
if its relational path has a length equal to 1. A \emph{relational
dependency} is of the form $[I_{j},\dotsc,I_{k}].Y\!\rightarrow\![I_{j}].X$
such that its cause and effect share the same base class and its effect
is canonical\emph{.}

Given a relational schema $\mathcal{S}$, a \emph{relational (causal)
model} $\mathcal{M}_{\Theta}$ is a pair of a \emph{structure} $\mathcal{M}=\left\langle \mathcal{S},\mathbf{D}\right\rangle $,
where $\mathbf{D}$ is the set of \emph{relational dependencies,}
and $\Theta$ is a set of parameters. We assume acyclicity of the
model so that the attribute classes can be partially ordered based
on $\mathbf{D}$. The parameters $\Theta$ define conditional distributions,
$p([I].X|\mbox{Pa}([I].X))$, for each pair $(I,X)$ where $I\in\mathcal{I}$,
$X\in\mathcal{A}(I)$, and $\mathrm{\mbox{Pa}}([I].X)$ is a set of
causes of $[I].X$, i.e., $\{P.Y|P.Y\!\rightarrow\![I].X\in\mathbf{D}\}$.
This paper focuses on the structure of RCM. Hence we often omit parameters
$\Theta$ from $\mathcal{M}$.

\subparagraph{Terminal Set and Ground Graph}

Because a skeleton is an instantiation of an underlying schema, a
\emph{ground graph} is an instantiation of the underlying RCM given
a skeleton translating relational dependencies to every entity and
relationship in the skeleton. It is obtained by interpreting the dependencies
defined by the RCM on the skeleton using the \emph{terminal sets }of
each of the instances in the skeleton.

Given a relational skeleton $\sigma$, the \emph{terminal set} of
a relational path $P$ given a base $b\in\sigma(P_{1})$, denoted
by $P|_{b}$, is the set of terminal items reachable from $b$ when
we traverse the skeleton along $P$. Formally, a terminal set $P|_{b}$
is defined recursively, $P^{1:1}|_{b}=\{b\}$ and 
\[
P^{1:\ell}|_{b}=\{i\in\sigma(P_{\ell})\mid j\!\in\!P^{1:\ell-1}|_{b},\,i\!\sim\!j\}\setminus{\textstyle \bigcup_{1\leq k<\ell}}P^{1:k}|_{b}.
\]
This implies that $P^{1:\ell}|_{b}$ and $P|_{b}$ will be disjoint
for $1\leq\ell<\left|P\right|$. Restricting the traversals so as
not to revisit any previously visited items corresponds to the \emph{bridge
burning semantics }(hereinafter, BBS) \citep{maier2013rds-workshop}.
The instantiation of an RCM $\mathcal{M}$ for a skeleton\emph{ $\sigma$}
yields a ground graph which we denote by $GG_{\mathcal{M}\sigma}$.
The vertices of $GG_{\mathcal{M}\sigma}$ are labeled by pairs of
items and its attribute, $\{i.X\mid I\!\in\!\mathcal{I},\,i\!\in\!\sigma(I),\,X\!\in\!\mathcal{A}(I)\}$.
There exists an edge $i_{j}.X\!\rightarrow\!i_{k}.Y$ in $GG_{\mathcal{M}\sigma}$
such that $i_{j}\!\in\!\sigma(I_{j})$, $i_{k}\!\in\!\sigma(I_{k})$,
$Y\!\in\!\mathcal{A}(I_{k})$, and $X\!\in\!\mathcal{A}(I_{j})$ if
and only if there exists a dependency $P.X\!\rightarrow\![I_{k}].Y\!\in\!\mathbf{D}$
such that $i_{k}\!\in\!P|_{i_{j}}$. 

In essence, RCM models dependencies on relational domain as follows:
Causal relationships are described from the perspective of each item
class; and are interpreted for each items to determine its causes
in a skeleton yielding a ground graph. Since an RCM is defined on
a given schema, RCM is interpreted on a skeleton so that every ground
graph is an instantiation of the RCM. 

Throughout this paper, unless specified otherwise, we assume a relational
schema $\mathcal{S}$, a set of relational dependencies $\mathbf{D}$,
and an RCM $\mathcal{M}=\left\langle \mathcal{S},\mathbf{D}\right\rangle $.

\section{\label{sec:REASONING-ON-RCM}REASONING WITH AN RCM}

An RCM can be seen as a \emph{meta} causal model or a \emph{template}
whose instantiation, a ground graph, corresponds to a \emph{traditional}
causal model (e.g., a causal Bayesian network). Reasoning with causal
models relies on \emph{conditional independence} (CI) relations among
variables. Graphical criteria such as \emph{d}-separation \citep{pearl2000causality}
are often exploited to test CI given a model. Hence, the traditional
definitions and methods for reasoning with causal models need to be
``lifted'' to the relational setting in order to be applicable to
\emph{relational} causal models.
\begin{defn}[Relational \emph{d}-separation \citep{maier2014thesis}]
Let $\mathbf{U}$, $\mathbf{V}$, and $\mathbf{W}$ be three disjoint
sets of relational variables with the same perspective $B\in\mathcal{I}$
defined over relational schema $\mathcal{S}$. Then, for relational
model structure $\mathcal{M},$ $\mathbf{U}$ and $\mathbf{V}$ are
\emph{d}-separated by $\mathbf{W}$ if and only if, for all skeletons
$\sigma\in\Sigma_{\mathcal{S}}$, $\mathbf{U}|_{b}$ and $\mathbf{V}|_{b}$
are \emph{d}-separated by $\mathbf{W}|_{b}$ in ground graph $GG_{\mathcal{M}\sigma}$
for all $b\in\sigma\left(B\right)$. 
\end{defn}
There are two things implicit in this definition: (i) \emph{all-ground-graphs
semantics} which implies that \emph{d}-separation must be hold over
\emph{all} instantiations of the model; (ii) the terminal set items
of two \emph{different} relational variables may overlap (which we
refer to as\emph{ intersectability}). In other words, two relational
variables $U=P.X$ and $V=P^{\prime}.X$ of the same perspective $B$
and the same attribute, are said to be \emph{intersectable} if and
only if:
\begin{equation}
\exists_{\sigma\in\Sigma_{\mathcal{S}}}\exists_{b\in\sigma\left(B\right)}P|_{b}\cap P^{\prime}|_{b}\neq\emptyset.\label{eq:intersectability-formula}
\end{equation}
In order to allow testing of conditional independence on all ground
graphs, \citet{maier2013rcd} introduced an \emph{abstract ground
graph} (AGG), which abstracts \emph{all} \emph{ground graphs} and
is able to cope with the \emph{intersectability} of relational variables.
We first recapitulate the original definition of AGGs.

\subsection{\label{sub:ORIGINAL-ABSTRACT-GROUND}ORIGINAL ABSTRACT GROUND GRAPHS}

An abstract ground graph $AGG_{\mathcal{M}B}$ is defined for a given
relational model $\mathcal{M}$ and a perspective $B\in\mathcal{I}$
\citep{maier2013rcd}, Since we fix the model, we omit the subscript
$\mathcal{M}$ and denote the abstract ground graph for perspective
$B$ by $AGG_{B}$. The resulting graph consists of two types of vertices:
$\mathbf{RV}_{B}$ and $\mathbf{IV}_{B}$; and two types of edges:
$\mathbf{RVE}_{B}$ and $\mathbf{IVE}_{B}$.

\begin{figure}
\centering\footnotesize\begin{tikzpicture}[scale=0.8, every node/.style={scale=0.8},every label/.style={yshift=-1.5em}, ->,>=stealth',shorten >=1pt,auto,node distance=3cm,semithick]
\tikzset{ent state/.style={semithick,draw,rectangle,aspect=1,fill=white,text=black,minimum width=1.7em, minimum height=1.7em}}
\tikzset{rel state/.style={draw,semithick,diamond,aspect=1,fill=white,text=black,minimum width=1.7em, minimum height=1.7em}}
\tikzstyle{every state}=[rectangle,fill=white,draw=none,text=black]

\node[ent state] (P1E1) at ($(0,0)$) [label={$E_3$}]{};
\node[rel state] (P1D1) at ($(P1E1)+(-0.7,0)$) [label={$R_b$}]{};
\node[ent state] (P1C1) at ($(P1D1)+(-0.7,0)$) [label={$E_2$}]{};
\node[rel state] (P1B1) at ($(P1C1)+(-0.7,0)$) [label={$R_a$}]{};
\node[ent state] (P1A1) at ($(P1B1)+(-0.7,0)$) [label={$E_1$}]{};
\node[rel state] (P1B2) at ($(P1E1)+(0.7,0)$) [label={$R_b$}]{};
\node[ent state] (P1C2) at ($(P1B2)+(0.7,0)$) [label={$E_2$}]{};
\node[rel state] (P1D2) at ($(P1C2)+(0.7,0)$) [label={$R_c$}]{};
\node[ent state] (P1E2) at ($(P1D2)+(0.7,0)$) [label={$E_4$}]{};
\draw[->] ($(P1A1.west)+(0,0.4)$) -- ($(P1E2.east)+(0,0.4)$) node[midway,above] {$P\Join_{1}Q$};
\draw[->] ($(P1A1.west)+(0,-0.4)$) -- ($(P1E1.east)+(0,-0.4)$) node[midway,below] {$P$};
\draw[->] ($(P1E1.west)+(0,-0.6)$) -- ($(P1E2.east)+(0,-0.6)$) node[midway,below] {$Q$};

\node[ent state] (P3C1) at ($(P1E1)+(5.5,0.5)$) [label={$E_2$}]{};
\node[rel state] (P3B1) at ($(P3C1)+(-0.7,0)$) [label={$R_a$}]{};
\node[ent state] (P3A1) at ($(P3B1)+(-0.7,0)$) [label={$E_1$}]{};
\node[rel state] (P3D2) at ($(P3C1)+(0.7,0)$) [label={$R_c$}]{};
\node[ent state] (P3E2) at ($(P3D2)+(0.7,0)$) [label={$E_4$}]{};
\node[rel state] (P3D1) at ($(P3C1)+(0,-0.7)$) [label={$R_b$}]{};
\node[ent state] (P3E1) at ($(P3D1)+(0,-0.7)$) [label={$E_3$}]{};
\draw[->] ($(P3A1.west)+(0,0.4)$) -- ($(P3E2.east)+(0,0.4)$) node[midway,above] {$P\Join_{3}Q$};
\draw[->] ($(P3A1.west)+(0,-0.4)$) to [bend left=45, looseness=2] node[midway,below left] {$P$} ($(P3E1.south)+(-0.4,0)$);
\draw[->] ($(P3E1.south)+(0.4,0)$) to [bend left=45, looseness=2] node[midway,below right] {$Q$} ($(P3E2.east)+(0,-0.4)$);

\end{tikzpicture}

\protect\caption{\label{fig:extend-method}A schematic example of how $\mathsf{extend}$
is computed showing two relational paths in $P\protect\Join Q$ where
$\ensuremath{\mathsf{card}(R_{b},E_{3})=\mathsf{many}}$. If $\ensuremath{\mathsf{card}(R_{b},E_{3})}$
is $\mathsf{one}$, then $P\protect\Join_{1}Q$ is not valid due to
rule 3 of relational path. A path $P\protect\Join_{2}Q$ is invalid
due to the violation of rule 2, i.e., $[\dotsc,E_{2},R_{b},E_{2},\dotsc]$.}
\end{figure}
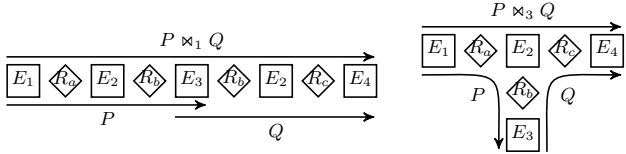
We denote by $\mathbf{RV}_{B}$ the set of \emph{all }relational variables
(RV) whose paths originate in $B$. We denote by $\mathbf{RVE}_{B}$
the set of all edges between the relational variables in $\mathbf{RV}_{B}$.
A relational variable edge (RVE) implies \emph{direct} influence arising
from one or more dependencies in $\mathbf{D}$. There is an RVE $P.X\!\rightarrow\!Q.Y$
if there exists a dependency $R.X\!\rightarrow\![I_{Y}].Y\in\mathbf{D}$
that can be interpreted as a direct influence from $P.X$ to $Q.Y$
from perspective $B$. Such an interpretation is implemented by an
$\mathsf{extend}$ function, which takes two relational paths and
produces a set of relational paths: If $P\in\mathsf{extend}(Q,R)$,
then there exists an RVE $P.X\!\rightarrow\!Q.Y$ where 
\begin{equation}
\mathsf{extend}(Q,R)=\{Q^{1:|Q|-i}+R^{i:}|i\in\mathsf{pivots}(\tilde{Q},R)\}\!\cap\!\mathbf{P}_{\mathcal{S}},\label{eq:extend}
\end{equation}
$\mathsf{pivots}(S,T)\!=\!\{i|S^{1:i}\!=\!T^{1:i}\}$, and `$+$'
is a concatenation operator. We will use a binary \emph{join} operator
`$\Join$' for $\mathsf{extend}$ and denote $Q^{1:|Q|-i}+R^{i:}$
by $Q\!\Join_{i}\!R$ for a pivot $i$. A schematic overview of $\mathsf{extend}$
is shown in \prettyref{fig:extend-method}. 

We denote by $\mathbf{IV}_{B}$ the set of \emph{intersection variables}
(IVs), i.e., unordered pairs of \emph{intersectable} relational variables
in $\mathbf{RV}_{B}$. Given two RVs $P.X$ and $P^{\prime}.X$ that
are intersectable with each other, we denote the resulting intersection
variable by $P.X\cap P^{\prime}.X$ (Here, the intersection symbol
`$\cap$' denotes \emph{intersectability }of the two relational
variables). By the definition \citep{maier2013rds-workshop}, if there
exists an RVE $P.X\!\rightarrow\!Q.Y$, then there exist edges $P.X\!\cap\!P^{\prime}.X\!\rightarrow\!Q.Y$
and $P.X\!\rightarrow\!Q.Y\!\cap\!Q^{\prime}.Y$ for every $P^{\prime}$
and $Q^{\prime}$ intersectable with $P$ and $Q$, respectively.
The IVs and the edges that connect them with RVs (IVEs) correspond
to \emph{indirect influences} (arising from intersectability) as opposed
to \emph{direct} influence due to dependencies (which are covered
by RVs and RVEs). We denote by $\mathbf{IVE}_{B}$ the set of all
such edges that connect RVs with IVs. 

Two AGGs with different perspectives share no vertices nor edges.
Hence, we view all AGGs, $\{AGG_{B}\}_{B\in\mathcal{I}}$, as a collection
or a single multi-component graph $\mathbf{AGG}=\bigcup_{B\in\mathcal{I}}AGG_{B}$.
We similarly define $\mathbf{RV}$, $\mathbf{IV}$, $\mathbf{RVE}$,
and $\mathbf{IVE}$ as the unions of their perspective-based counterparts.

\begin{figure}
\center\scriptsize\scalebox{0.95}{\begin{tikzpicture}[scale=0.9,->,>=stealth',shorten >=1pt,auto,node distance=2em,
                    semithick]
  \tikzstyle{every node}=[scale=0.9,rounded corners=3pt,rectangle,fill=none,draw=black,text=black]
\node                              (E)                      {$\mathsf{[Emp].comp}$};
\node                              (EDP)  [right=2em of E]  {$\mathsf{[Emp,Dev,Prod].succ}$};
\node[circle,fill=white,scale=0.7] (XEDP) at (EDP.north west) {$V$};
\node[very thick]                  (EDPFB) [right=2em of EDP] {$\mathsf{[Emp,Dev,Prod,Fund,Biz].rev}$};
\node[circle,fill=white,scale=0.7] (XEDPFB) at (EDPFB.north west) {$W$};

\node[very thick]                  (EDPDE) at ($(EDP)+(-1.6,-1)$) {$\mathsf{[Emp,Dev,Prod,Dev,Emp].comp}$};
\node[circle,fill=white,scale=0.7] (XEDPDE) at (EDPDE.north west) {$X$};
\node                              (EDPDEDP) at ($(EDPDE)+(1,-1)$) {$\mathsf{[Emp,Dev,Prod,Dev,Emp,Dev,Prod].succ}$};
\node[circle,fill=white,scale=0.7] (XEDPDEDP) at (EDPDEDP.north west) {$Z$};
\node                              (EDPFBFP) at ($(EDPFB)+(-1,1)$) {$\mathsf{[Emp,Dev,Prod,Fund,Biz,Fund,Prod].succ}$};
\node[circle,fill=white,scale=0.7] (XEDPFBFP) at (EDPFBFP.north west) {$U$};
\node              (IV) at ($(EDPFB)+(-0.5,-1)$) {$\mathsf{\substack{\mathsf{[Emp,Dev,Prod,Dev,Emp,Dev,Prod].succ}\\ \cap \\ \mathsf{[Emp,Dev,Prod,Fund,Biz,Fund,Prod].succ}}}$};
\node[circle,fill=white,scale=0.7] (XIV) at (IV.north west) {$Y$};

\draw[->] (E.east) -- (EDP.west);
\draw[->] (EDP.east) -- (EDPFB.west);
\draw[->] (EDPDE.east) -- (IV.west);
\draw[->] (EDPDE.south) -- (EDPDEDP.north);
\draw[->] (EDPDE.north) -- (EDP.south);
\draw[->] (IV.north) -- (EDPFB.south);
\draw[->] (EDPFBFP.south) -- (EDPFB.north);

\end{tikzpicture}}

\protect\caption{\label{fig:agg_example}An AGG example excerpted from \citet{maier2014thesis}
with \emph{business unit }($\mathsf{Biz}$) which \emph{funds} ($\mathsf{Fund}$)
its \emph{products} from its revenue ($\mathsf{rev}$). The revenue
of business units that fund the products developed by an employee
($W$) is affected by the employee's co-workers' competence ($X$),
i.e., $\bar{W}\!\protect\not\Perp\!\bar{X}$. Two are conditionally
independent by blocking both $V$ and $Y$. Since IV $Y$ is in $\bar{U}$
and $\bar{Z}$, both $\bar{W}\!\Perp\!\bar{X}|\overline{\{V,U\}}$
and $\bar{W}\!\Perp\!\bar{X}|\overline{\{V,Z\}}$ hold, which are
equivalent to $\left(W\!\perp\!X|\{V,U\}\right)_{\mathcal{M}}$ and
$\left(W\!\perp\!X|\{V,Z\}\right)_{\mathcal{M}}$, respectively.}
\end{figure}
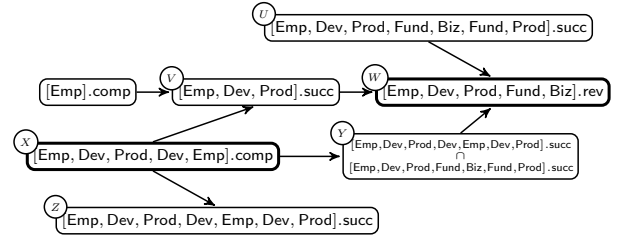

For any mutually disjoint sets of relational variables $\mathbf{U},\mathbf{V}$,
and $\mathbf{W}$, one can test $\mathbf{U}\perp\mathbf{V}\mid\mathbf{W}$,
conditional independence admitted by the underlying probability distribution,
by checking $\bar{\mathbf{U}}\Perp\bar{\mathbf{V}}\mid\bar{\mathbf{W}}$
(traditional)\emph{ d}-separation on an AGG\footnote{We denote conditional independence by `$\perp$' in general. We use
`$\Perp$' to represent (traditional) \emph{d}-separation on a directed
acyclic graph., e.g., $\mathbf{AGG}_{\mathcal{M}}$ or $GG_{\mathcal{M}\sigma}$.
Furthermore, we parenthesize conditional independence and use a subscript
to specify the scope of the conditional independence, if necessary.}, where $\mathbf{\bar{V}}$ includes $\mathbf{V}$ and their related
IVs, $\bar{\mathbf{V}}\!=\!\mathbf{V}\!\cup\!\{V\!\cap\!T\!\in\!\mathbf{IV}\mid V\!\in\!\mathbf{V}\}.$
\prettyref{fig:agg_example} illustrates relational \emph{d}-separation
on an AGG. 

We later show that the preceding definition of AGG does \emph{not}
properly abstract all ground graphs; nor does it guarantee the correctness
of reasoning about relational \emph{d}-separation in an RCM. We revise
the definition of AGG (\prettyref{sub:Revision-of-Abstract}) so as
to ensure that the resulting AGG abstracts all ground graphs. However,
we find that even with the revised definition of AGG, the AGG representation
is \emph{not complete} for relational \emph{d}-separation\emph{,}
that is, there can exist conditional independence relations in an
RCM that are not entailed by AGG (\prettyref{sub:COUNTEREXAMPLE}).
A careful examination of the lack of completeness of AGG for relational
\emph{d}-separation with respect to causal faithfulness yields useful
insights that allow us to make use of weaker notions of faithfulness
to learn RCM from data (\prettyref{sub:RELATING-NON-COMPLETENESS-WITH}).

\subsection{\label{sub:Revision-of-Abstract} ABSTRACT GROUND GRAPHS - A REVISED
DEFINITION}

Because of the importance of $\mathbf{IV}$ and $\mathbf{IVE}$ in
$\mathbf{AGG}$ in reasoning about relational \emph{d}-separation,
it is possible that errors in abstracting all ground graphs could
lead to errors in CI relations inferred from an $\mathbf{AGG}$. We
proceed to show that 1) the criteria for determining \emph{intersectability}
\citep{maier2014thesis} are \emph{not} \emph{sufficient,} and 2)
the definition of $\mathbf{IVE}$, as it stands, does not guarantee
the \emph{soundness} of AGG as an abstract representation of the all
ground graphs of an RCM. We provide the\emph{ necessary and sufficient}
criteria for determining IVs and a \emph{sound} definition for IVEs.

\subsubsection{\label{sub:Intersectability-and-IV}Intersectability and IV}

The declarative characterization of \emph{intersectability} (\prettyref{eq:intersectability-formula})
does not offer practical procedural criteria to determine \emph{intersectability}.
Based on the criteria \citep{maier2014thesis}, two different relational
paths $P$ and $Q$ are \emph{intersectable} if and only if 1) they
share the same perspective, say $B\in\mathcal{I}$, and 2) they share
the common terminal class, and 3) one path is \emph{not} a prefix
of the other. We will prove that the preceding criteria are \emph{not}
\emph{sufficient}. In essence, we will show the conditions under which
non-emptiness of $P|_{b}\cap Q|_{b}$ for any $b\in\sigma(B)$ in
any skeleton $\sigma$ always contradicts the BBS. For the proof,
we define $\mathsf{LLRSP}(P,Q)$ (\emph{the length} of \emph{the longest
required shared path}) for two relational paths $P$ and $Q$ of the
common perspective as 
\[
\max\{\ell\mid P^{1:\ell}=Q^{1:\ell},\,\forall_{\sigma\in\Sigma_{\mathcal{S}}}\forall_{b\in\sigma\left(B\right)}\left|P^{1:\ell}|_{b}\right|=1\}.
\]
$\mathsf{LLRSP}(P,Q)$ is computed as follows. Initially set $\ell\!=\!1$
since $P_{1}\!=\!Q_{1}$. Repeat incrementing $\ell$ by $1$ if $P_{\ell+1}=Q_{\ell+1}$
and either $P_{\ell}\in\mathcal{R}$ or $P_{\ell}\in\mathcal{E}$
with $\mathsf{card}(P_{\ell},P_{\ell+1})=\mathsf{one}$.
\begin{lem}
\label{lem:intersectability-nec-suf}Given a relational schema $\mathcal{S}$,
let $P$ and $Q$ be two different relational paths satisfying the
(necessary) criteria of \citet{maier2014thesis} and $\left|Q\right|\leq\left|P\right|$.
Let $m$ and $n$ be $\mathsf{LLRSP}(P,Q)$ and $\mathsf{LLRSP}(\tilde{P},\tilde{Q})$,
respectively. Then, $P$ and $Q$ are intersectable if and only if
$m+n\leq\left|Q\right|$.\end{lem}
\begin{proof}
See Appendix.
\end{proof}
The lemma demonstrates the criteria by \citet{maier2014thesis} do
not rule out the case of $m+n>\left|Q\right|$ where $P$ and $Q$
cannot be intersectable.

\subsubsection{\label{sub:Co-intersectability-and-IVE}Co-intersectability and IVE}

Based on the definition \citep{maier2014thesis}, an IVE exists between
an IV, $U\!\cap\!V$, and an RV, $W$, if and only if there exists
an RVE between $U$ and $W$ or $V$ and $W$. It would indeed be
appealing to define IV, $U\!\cap\!V$, such that it inherits properties
of the corresponding RVs, $U$ and $V$. However, the abstract ground
graph resulting from such a definition turns out to be not a sound
representation of the underlying ground graphs. We proceed to prove
this result.
\begin{figure}
\begin{eqnarray*}
1)\,\exists_{\sigma\in\Sigma_{\mathcal{S}}}\exists_{b\in\sigma\left(B\right)}\exists_{i_{j}\in Q|_{b}} & R|_{i_{j}}\cap P|_{b}\quad\qquad\neq\emptyset\\
2)\,\exists_{\sigma\in\Sigma_{\mathcal{S}}}\exists_{b\in\sigma\left(B\right)}\quad\qquad & \quad\qquad P|_{b}\cap P^{\prime}|_{b}\neq\emptyset\\
3)\,\exists_{\sigma\in\Sigma_{\mathcal{S}}}\exists_{b\in\sigma\left(B\right)}\exists_{i_{j}\in Q|_{b}} & \,R|_{i_{j}}\cap P|_{b}\cap P^{\prime}|_{b}\neq\emptyset
\end{eqnarray*}

\protect\caption{\label{fig:Comparison-of-conditions}Comparison of 1) the necessary
condition of the existence of an RVE $P\rightarrow Q$ through $R$,
the cause path of a dependency (attributes are omitted), 2) intersectability
between $P$ and $P^{\prime}$, and 3) co-intersectability of $\langle Q,R,P,P^{\prime}\rangle$.}
\end{figure}

\begin{defn}[Co-intersectability]
\label{def:co-inter}Given a relational schema $\mathcal{S}$, let
$Q$, $R$, $P$, and $P^{\prime}$ be valid relational paths of the
same perspective $B$ where $P\in Q\!\Join\!R$ and $P$ and $P^{\prime}$
are intersectable. Then, a tuple $\langle Q,R,P,P^{\prime}\rangle$
is said to be \emph{co-intersectable} if and only if 
\begin{equation}
\exists_{\sigma\in\Sigma_{\mathcal{S}}}\exists_{b\in\sigma\left(B\right)}\exists_{i_{j}\in Q|_{b}}\,R|_{i_{j}}\cap P|_{b}\cap P^{\prime}|_{b}\neq\emptyset.\label{eq:co-inter}
\end{equation}
We relate co-intersectability with the definition of IVE. Let an RVE
$P.X\!\rightarrow\!Q.Y$ is due to some dependencies $R.X\!\rightarrow\![I_{Y}].Y\in\mathbf{D}$
where $P\in Q\!\Join\!R$. This implies 
\begin{equation}
\exists_{\sigma\in\Sigma_{\mathcal{S}}}\exists_{b\in\sigma\left(B\right)}\exists_{i_{j}\in Q|_{b}}\,R|_{i_{j}}\cap P|_{b}\neq\emptyset,\label{eq:direct-influence}
\end{equation}
and there are edges from $X$ of $R|_{i_{j}}\!\cap\!P|_{b}$ to $Y$
of $Q|_{b}$ in $GG_{\sigma}$. In order for the intersectability
of $P^{\prime}$ with $P$ translates into an influence between $P$
and $Q$, it is necessary that there exists a skeleton that admits
such influence. However, we can construct a counterexample that satisfies
the necessary conditions for the existence of an RVE and the conditions
for intersectability but does \emph{not} satisfy the conditions for
co-intersectability (see \prettyref{fig:Comparison-of-conditions}
for a comparison of \prettyref{eq:direct-influence}, \ref{eq:intersectability-formula},
and \ref{eq:co-inter}).
\end{defn}
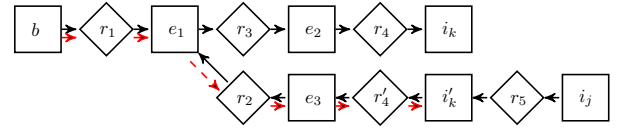
\begin{figure}
\center\begin{tikzpicture}[scale=0.7, every node/.style={scale=0.7},->,>=stealth',shorten >=1pt,auto,node distance=1.3cm,semithick]
 \tikzstyle{every state}=[rectangle, fill=white,draw=black,text=black,text width=7pt, text height=7pt]

\node[state]          (B)                  {$b$};
\node[state, diamond] (R1)  [right of=B]   {$r_1$};
\node[state]          (E1)  [right of=R1]  {$e_1$};
\node[state, diamond] (R3)  [right of=E1]  {$r_3$};
\node[state]          (E2)  [right of=R3]  {$e_2$};
\node[state, diamond] (R4)  [right of=E2]  {$r_4$};
\node[state]          (Ik)  [right of=R4]  {$i_k$};
\node[state, diamond] (R2)  [below of=R3]  {$r_2$};
\node[state]          (E3)  [right of=R2]  {$e_3$};
\node[state, diamond] (R42) [right of=E3]  {$r_4^\prime$};
\node[state]          (Ik2) [right of=R42] {$i_k^{\prime}$};
\node[state, diamond] (R5)  [right of=Ik2] {$r_5$};
\node[state]          (Ij)  [right of=R5]  {$i_j$};

 \path
(B)   edge     node {$$} (R1)
(R1)  edge     node {$$} (E1)
(E1)  edge     node {$$} (R3)
(R3)  edge     node {$$} (E2)
(E2)  edge     node {$$} (R4)
(R4)  edge     node {$$} (Ik)
(E1)  edge[<-] node {$$} (R2)
(R2)  edge[<-] node {$$} (E3)
(E3)  edge[<-] node {$$} (R42)
(R42) edge[<-] node {$$} (Ik2)
(Ik2) edge[<-] node {$$} (R5)
(R5)  edge[<-] node {$$} (Ij)

([yshift=-1ex]B.east) edge [dashed,draw=red] node {$$} ([yshift=-1ex]R1.west)
([yshift=-1ex]R1.east) edge[dashed,draw=red] node {$$} ([yshift=-1ex]E1.west)
([xshift=-1ex,yshift=-1ex]E1.south east) edge[dashed,draw=red] node {$$} ([yshift=-1ex,xshift=-1ex]R2.north west)
([yshift=-1ex]R2.east) edge[dashed,draw=red] node {$$} ([yshift=-1ex]E3.west)
([yshift=-1ex]E3.east) edge[dashed,draw=red] node {$$} ([yshift=-1ex]R42.west)
([yshift=-1ex]R42.east) edge[dashed,draw=red] node {$$} ([yshift=-1ex]Ik2.west);
\end{tikzpicture}

\protect\caption{\label{fig:counterexample-for-cointer}A schematic illustration of
Example \ref{example:cointersectability-counterexample} superimposing
a skeleton and relational paths. The items for $P^{\prime}$ starting
with $b$ should follow a dashed red line, and, hence, $P^{\prime}$
cannot be related to a ground graph edge between $i_{k}$ and $i_{j}$,
i.e., an RVE between $P$ and $Q$ (attributes and connections between
entities and relationships are omitted).}
\end{figure}

\begin{example}
\label{example:cointersectability-counterexample}Let $\mathcal{S}$
be a relational schema where $\mathbf{E}=\left\{ I_{j},I_{k},B,E_{1},E_{2},E_{3}\right\} $,
$\mathbf{R}=\left\{ R_{i}\right\} _{i=1}^{5}$ such that $R_{1}=\left\langle B,E_{1}\right\rangle $,
$R_{2}=\left\langle E_{1},E_{3}\right\rangle $, $R_{3}=\left\langle E_{1},E_{2}\right\rangle $,
$R_{4}=\left\langle E_{2},E_{3},I_{k}\right\rangle $, $R_{5}=\left\langle I_{k},I_{j}\right\rangle $
with the cardinality of each relationship and each entity in the relationship
being $\mathsf{one}$. Let 
\begin{itemize}
\item \itemsep0em  $Q=\left[B,R_{1},E_{1},R_{2},E_{3},R_{4},I_{k},R_{5},I_{j}\right]$,
\item $R=\left[I_{j},R_{5},I_{k},R_{4},E_{3},R_{2},E_{1},R_{3},E_{2},R_{4},I_{k}\right]$, 
\item $P=\left[B,R_{1},E_{1},R_{3},E_{2},R_{4},I_{k}\right]$, and
\item $P^{\prime}=\left[B,R_{1},E_{1},R_{2},E_{3},R_{4},I_{k}\right]$. 
\end{itemize}
Observe that
\begin{enumerate}
\item \itemsep0em  $P\in\mathsf{extend}\left(Q,R\right)$;
\item $P^{\prime}$ and $P$ are intersectable; and 
\item $P^{\prime}$ is a subpath of $Q$.
\end{enumerate}
This example satisfies \prettyref{eq:intersectability-formula} and
\prettyref{eq:direct-influence}. Assume for contradiction that there
exists a skeleton $\sigma$ satisfying \prettyref{eq:co-inter}. Since,
in this example, the cardinality of each relationship and each entity
in the relationship is $\mathsf{one}$, for each $b\in\sigma\left(B\right)$,
there exists only one $i_{j}\in Q|_{b}$ and only one $i_{k}\in P|_{b}$.
By the assumption, $P^{\prime}|_{b}=\left\{ i_{k}\right\} $. Since
$P^{\prime}$ is a subpath of $Q$, $P^{\prime}|_{b}$ will end at
$i_{k}^{\prime}=R^{1:3}|_{i_{j}}$ (see \prettyref{fig:counterexample-for-cointer}).
Due to BBS, $R|_{i_{j}}\cap R^{1:3}|_{i_{j}}=\emptyset$, that is,
$\left\{ i_{k}\right\} \cap\left\{ i_{k}^{\prime}\right\} =\emptyset$.
This contradicts the assumption that $i_{k}=i_{k}^{\prime}$.
\end{example}
This counterexample clearly represents there is an inter-dependency
between intersection variables and RVEs. Therefore, we revise the
definition of $\mathbf{IVE}$ accompanying co-intersectability.
\begin{defn}[IVE]
\label{def:IVE}There exists an IVE edge, $P.X\!\cap\!P^{\prime}.X\!\rightarrow\!Q.Y$
(or $P.X\!\rightarrow\!Q.Y\!\cap\!Q^{\prime}.Y$), if and only if
there exists a relational path $R$ such that $R.X\!\rightarrow\![I_{Y}].Y\in\mathbf{D},$
$P\in Q\!\Join\!R$, and $\left\langle Q,R,P,P^{\prime}\right\rangle $
(or $\langle P,\tilde{R},Q,Q^{\prime}\rangle$) is \emph{co-intersectable}.
\end{defn}
To determine IVEs, \emph{co-intersectability} of a tuple can be computed
by solving a constraint satisfaction problem involving four paths
in the tuple.

\paragraph{Implications of Co-intersectability}

We investigated the necessary and sufficient criteria for intersectability
and revised the definition of IVE so as to guarantee that AGG correctly
abstracts all ground graphs as asserted (although incorrectly) by
Theorem 4.5.2 \citep{maier2014thesis}. The new criterion, called
\emph{co-intersectability,} is especially interesting since it describes
the interdependency between intersection variables and related relational
variable edges. Several of the key results (e.g., soundness and completeness
of AGG for relational \emph{d}-separation, Theorem 4.5.2) and concepts
(e.g., ($B$,$h$)-reachability) of \citet{maier2014thesis} are based
on \emph{independence} between intersection variables and related
relational variable edges. Hence, it is useful to carefully scrutinize
the relationship between AGG and relational \emph{d}-separation.

\section{\label{sec:NON-COMPLETENESS-OF-AGG}NON-COMPLETENESS OF AGG FOR RELATIONAL
D-SEPARATION}

We first revisit the definition of relational \emph{d}-separation.
Given three disjoint sets of relational variables $\mathbf{U}$, $\mathbf{V}$,
and $\mathbf{W}$ of a common perspective $B\in\mathcal{I}$, $\mathbf{U}$
and $\mathbf{V}$ are relational \emph{d}-separated given $\mathbf{W}$,
denoted by $\left(\mathbf{U}\!\perp\!\mathbf{V}\mid\mathbf{W}\right)_{\mathcal{M}}$,
if and only if 
\[
\forall_{\sigma\in\Sigma_{\mathcal{S}}}\forall_{b\in\sigma\left(B\right)}\left(\mathbf{U}|_{b}\Perp\mathbf{V}|_{b}\mid\mathbf{W}|_{b}\right)_{GG_{\mathcal{M}\sigma}}.
\]
From Theorem 4.5.4 of \citep{maier2014thesis}, the lifted representation
$\mathbf{AGG}_{\mathcal{M}}$ is said to be sound (or complete) for
relational \emph{d}-separation of $\mathcal{M}$ if (traditional)
\emph{d}-separation holds on the $\mathbf{AGG}_{\mathcal{M}}$ with
a modified CI query only when (or whenever) relational \emph{d}-separation
holds true. Then, the completeness of AGG for relational \emph{d}-separation
can be represented as
\[
\left(\mathbf{U}\perp\mathbf{V}\mid\mathbf{W}\right)_{\mathcal{M}}\Rightarrow\left(\bar{\mathbf{U}}\Perp\bar{\mathbf{V}}\mid\bar{\mathbf{W}}\right)_{\mathbf{AGG}_{\mathcal{M}}}.
\]
The completeness can be proved by the construction of a skeleton $\sigma\in\Sigma_{\mathcal{S}}$
demonstrating \emph{d}-connection $\left(\mathbf{U}|_{b}\not\Perp\mathbf{V}|_{b}\mid\mathbf{W}|_{b}\right)_{GG_{\mathcal{M}\sigma}}$
for some $b\in\sigma\left(B\right)$ if $(\bar{\mathbf{U}}\not\Perp\bar{\mathbf{V}}\mid\bar{\mathbf{W}})_{\mathbf{AGG}_{\mathcal{M}}}$.
In other words, we might disprove the completeness by showing

\begin{eqnarray*}
 & (\bar{\mathbf{U}}\!\not\Perp\!\bar{\mathbf{V}}\mid\bar{\mathbf{W}})_{\mathbf{AGG}_{\mathcal{M}}}\wedge\qquad\qquad\qquad\qquad\qquad\\
 & \qquad\qquad\forall_{\sigma\in\Sigma_{\mathcal{S}}}\forall_{b\in\sigma\left(B\right)}\left(\mathbf{U}|_{b}\!\Perp\!\mathbf{V}|_{b}\mid\mathbf{W}|_{b}\right)_{GG_{\mathcal{M}\sigma}}.
\end{eqnarray*}

\subsection{\label{sub:COUNTEREXAMPLE}A COUNTEREXAMPLE}

The following counterexample shows that AGG is not complete for relational
\emph{d}-separation.
\begin{example*}
Let $\mathcal{S}=\left\langle \mathcal{E},\mathcal{R},\mathcal{A},\mathsf{card}\right\rangle $
be a relational schema such that: $\mathcal{E}=\left\{ E_{i}\right\} _{i=1}^{5}$;
$\mathcal{R}=\left\{ R_{j}\right\} _{j=1}^{3}$ with $R_{1}=\left\langle E_{1},E_{2},E_{4}\right\rangle $,
$R_{2}=\left\langle E_{2},E_{3}\right\rangle $, and $R_{3}=\left\langle E_{3},E_{4},E_{5}\right\rangle $;
$\mathcal{A}=\left\{ E_{2}:\left\{ Y\right\} ,\,E_{3}:\left\{ X\right\} ,\,E_{5}:\left\{ Z\right\} \right\} $;
and $\forall_{R\in\mathcal{R}}\forall_{E\in R}\mathsf{card}\left(R,E\right)=\mathsf{one}$.
Let $\mathcal{M}=\left\langle \mathcal{S},\mathbf{D}\right\rangle $
be a relational model with 
\[
\mathbf{D}=\left\{ D_{1}.X\rightarrow[I_{Y}].Y,\,D_{2}.Z\rightarrow[I_{Y}].Y\right\} 
\]
such that $D_{1}=\left[E_{2},R_{2},E_{3},R_{3},E_{4},R_{1},E_{2},R_{2},E_{3}\right]$
and $D_{2}=\left[E_{2},R_{2},E_{3},R_{3},E_{5}\right]$. Let $P.X$,
$Q.Y$, $S.Z$, and $S^{\prime}.Z$ be four relational variables of
the same perspective $B=E_{1}$ where their relational paths are distinct
where
\begin{itemize}
\item \itemsep0em  $P=\left[E_{1},R_{1},E_{2},R_{2},E_{3}\right]$,
\item $Q=\left[E_{1},R_{1},E_{4},R_{3},E_{3},R_{2},E_{2}\right]$,
\item $S=\left[E_{1},R_{1},E_{4},R_{3},E_{5}\right]$, and
\item $S^{\prime}=\left[E_{1},R_{1},E_{2},R_{2},E_{3},R_{3},E_{5}\right]$.
\end{itemize}
\end{example*}

Given the above example, we can make two claims.

\begin{figure}
\scriptsize\center\begin{tikzpicture}[->,>=stealth',auto,node distance=1cm,
                    semithick]
\tikzstyle{every state}=[rectangle,ultra thick, fill=none,draw=black,text=black]

\node[state]         (G)                    {$b$};
\node[state,diamond] (H) [left of=G]        {$r_1$};
\node[state,thick]   (I) [above of=H]       {$e_4$};
\node[state,diamond] (J) [above of=I]       {$r_3$};
\node[state]         (K) [right of=J]       {$i_z$};

\node[state]         (C) [left of=J]        {$e_3$};
\node[state,diamond] (B) [left of=C]        {$r_2^\prime$};
\node[state]         (A) [left of=B]        {$i_y$};

\draw[-] (G) -- (H) -- (I) -- (J) -- (K);
\draw[-] (J) -- (C) -- (B) -- (A) -- (H);
\draw[-,ultra thick] (G) -- (H) -- (A) -- (B) -- (C)--(J)--(K);
\draw[->,red,dashed] ($(K.north)$) to [out=135, in=45,looseness=0.5] ($(A.north)$);

\end{tikzpicture}

\protect\caption{\label{fig:co-intersectability}Co-intersectability of $\left\langle Q,D_{2},S,S^{\prime}\right\rangle $
where $i_{y}\in Q|_{b}$, $i_{z}\in D_{2}|_{i_{y}}$, $i_{z}\in S|_{b}$,
and $i_{z}\in S^{\prime}|_{b}$. The thick line highlights items for
$S^{\prime}$ from $b$ to $i_{z}$. The red dashed line represents
the instantiation of an RVE $S.Z\rightarrow Q.Y$ as $i_{z}.Z\rightarrow i_{y}.Y$
in a ground graph (attributes are omitted). }
\end{figure}
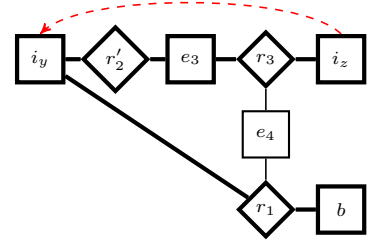

\begin{claim}
$\left(\overline{P.X}\not\Perp\overline{S^{\prime}.Z}\mid\overline{Q.Y}\right)_{\mathbf{AGG}_{\mathcal{M}}}$.\end{claim}
\begin{proof}
See Appendix.
\end{proof}
Assuming that AGG is complete for relational \emph{d}-separation,
we can infer $\left(P.X\not\perp S^{\prime}.Z\mid Q.Y\right)_{\mathcal{M}}$
and there must exist a pair of a skeleton $\sigma$ and a base $b\in\sigma\left(B\right)$
that satisfies $\left(P.X|_{b}\not\Perp S^{\prime}.Z|_{b}\mid Q.Y|_{b}\right)_{GG_{\mathcal{M}\sigma}}$.
However, we claim that such a skeleton and base may not exist.

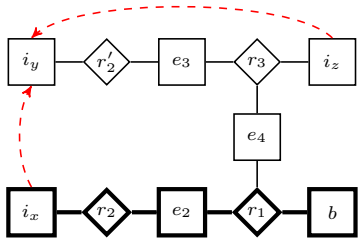
\begin{figure}
\scriptsize\center\begin{tikzpicture}[auto,->,>=stealth',node distance=1cm,
                    semithick]
\tikzstyle{every state}=[inner sep=0pt,rectangle,fill=none,draw=black,text=black]

\node[state,ultra thick]          (G)                    {$b$};
\node[state,diamond,ultra thick]  (H) [left of=G]        {$r_1$};
\node[state]                      (I) [above of=H]       {$e_4$};
\node[state,diamond]              (J) [above of=I]       {$r_3$};
\node[state]                      (K) [right of=J]       {$i_z$};

\node[state]                      (C) [left of=J]        {$e_3$};
\node[state,diamond]              (B) [left of=C]        {$r_2^\prime$};
\node[state]                      (A) [left of=B]        {$i_y$};

\node[state,ultra thick]          (F) [left of=H]        {$e_2$};
\node[state,,diamond,ultra thick] (E) [left of=F]        {$r_2$};
\node[state,ultra thick]          (D) [left of=E]        {$i_x$};

\draw[-] (G) -- (H) -- (I) -- (J) -- (K);
\draw[-] (J) -- (C) -- (B) -- (A);
\draw[-] (H) -- (F) -- (E) -- (D);

\draw[-,ultra thick] (G) -- (H) -- (F) -- (E) -- (D);
\draw[->,red,dashed] ($(K.north)$) to [out=135, in=45,looseness=0.5] ($(A.north)$);
\draw[->,red,dashed] ($(D.north)$) to [out=120, in=240,looseness=0.8] ($(A.south)$);

\end{tikzpicture}

\protect\caption{\label{fig:counterexample-for-completeness}A subgraph of ground graphs
to represent $i_{x}\rightarrow i_{y}\leftarrow i_{z}$. Only this
substructure satisfies BBS assumption and cardinality constraints. }
\end{figure}

\begin{claim}
There is no $\sigma\in\Sigma_{\mathcal{S}}$ and $b\in\sigma\left(B\right)$
such that 
\[
\left(P.X|_{b}\not\Perp S^{\prime}.Z|_{b}\mid Q.Y|_{b}\right)_{GG_{\mathcal{M}\sigma}}.
\]
\end{claim}
\begin{proof}
See Appendix.
\end{proof}
The counterexample demonstrates that a \emph{d}-connection path captured
in an $\mathbf{AGG}_{\mathcal{M}}$ might not have a corresponding
\emph{d}-connection path in \emph{any} ground graph.
\begin{cor}
\label{cor:noncompleteness}The revised (as well as the original)
abstract ground graph for an RCM is not complete for relational d-separation.
\end{cor}
It is possible that an additional test can be utilized to check whether
there \emph{exists} such a ground graph that can represent a \emph{d}-connection
path captured in $\mathbf{AGG}_{\mathcal{M}}$. However, the efficiency
of such an additional test is unknown and designing such a test is
beyond the scope of this paper.

\subsection{\label{sub:RELATING-NON-COMPLETENESS-WITH}RELATING NON-COMPLETENESS
WITH FAITHFULNESS}

In light of the preceding result that AGG is not complete for relational
\emph{d}-separation, we proceed to examine the relationship between
an RCM and its lifted representation in terms of the sets of conditional
independence relationships that they admit. In RCM, there are several
levels of relationship regarding the sets of conditional independence:
between the underlying probability distributions and the ground graphs,
between the ground graphs of an RCM and the RCM, and between the RCM
and its AGG: 
\[
\{p\leftrightarrow GG_{\mathcal{M}_{\Theta}\sigma}\}_{\sigma\in\Sigma_{\mathcal{S}}}\leftrightarrow\mathcal{M}_{\Theta}\leftrightarrow\mathbf{AGG}_{\mathcal{M}}
\]
In RCM, the \emph{causal Markov condition} and \emph{causal faithfulness
condition} (see below) can be applied between a ground graph $GG_{\mathcal{M}\sigma}$
and its underlying probability distribution $p$. Both conditions
are assumed for learning an RCM from relational data. Relational \emph{d}-separation
requires a set of conditional independence of $\mathcal{M}_{\Theta}$
using those deduced from every ground graph $GG_{\mathcal{M}_{\Theta}\sigma}$
for every $\sigma\in\Sigma_{\mathcal{S}}$. In light of the lack of
completeness of AGG for relational \emph{d}-separation, the set of
conditional independence relations admitted by $\mathcal{M}_{\Theta}$
and its lifted representation $\mathbf{AGG}_{\mathcal{M}}$ are \emph{not}
necessarily equivalent (see \prettyref{cor:noncompleteness}).

We will relate $\mathcal{M}$ and $\mathbf{AGG}_{\mathcal{M}}$ using
an \emph{analogy} of \emph{causal Markov condition} and \emph{faithfulness}
\citep{spirtes2000causation,Ramsey2006} interpreting $\mathbf{AGG}_{\mathcal{M}}$
and $\mathcal{M}$ as a DAG $G$ and a distribution $p$, respectively.
We first recapitulate the definitions for causal Markov condition
and faithfulness.
\begin{defn}[Causal Markov Condition \citep{Ramsey2006}]
Given a set of variables whose causal structure can be represented
by a DAG $G$, every variable is probabilistically independent of
its non-effects (non-descendants in $G$) conditional on its direct
causes (parents in $G$).
\end{defn}
The causal Markov condition (i.e., local Markov condition) is not
directly translated into the relationship between $\mathbf{AGG}_{\mathcal{M}}$
and $\mathcal{M}$ since they refer to different variables. However,
the soundness of $\mathbf{AGG}_{\mathcal{M}}$ for relational \emph{d}-separation
of $\mathcal{M}$ (i.e., global Markov condition) would be sufficient
to interpret causal Markov condition between $\mathbf{AGG}_{\mathcal{M}}$
and $\mathcal{M}$. That is,
\[
\forall_{U,V,W\in\mathbf{RV}}\left(\bar{U}\Perp\bar{V}\mid\bar{W}\right)_{\mathbf{AGG}_{\mathcal{M}}}\Rightarrow\left(U\perp V\mid W\right)_{\mathcal{M}}
\]
where $U$, $V$, and $W$ are distinct relational variables sharing
a common perspective.

\begin{defn}[Causal Faithfulness Condition \citep{Ramsey2006}]
Given a set of variables whose causal structure can be represented
by a DAG, no conditional independence holds unless entailed by the
causal Markov condition.
\end{defn}
By the counterexample above, $\mathcal{M}$ is not strictly \emph{faithful}
to $\mathbf{AGG}_{\mathcal{M}}$, because more conditional independences
hold in $\mathcal{M}$ than those entailed by $\mathbf{AGG}_{\mathcal{M}}$.

\subsubsection{Weaker Faithfulness Conditions}

\citet{Ramsey2006} showed that the two weaker types of faithfulness
-- \emph{adjacency-faithfulness} and \emph{orientation-faithfulness}
-- are sufficient to retrieve a maximally-oriented causal structure
from a data under the causal Markov condition. What we have showed
is that there are more conditional independence hold in $\mathcal{M}$
than those entailed by its corresponding $\mathbf{AGG}_{\mathcal{M}}$.
However, the two weaker faithfulness conditions hold true (if they
are appropriately interpreted in an RCM and its lifted representation).

\paragraph{Adjacency-Faithfulness}
\begin{defn}[Adjacency-Faithfulness \citep{Ramsey2006}]
Given a set of variables $\mathbf{V}$ whose causal structure can
be represented by a DAG $G$, if two variables $X$, $Y$ are adjacent
in $G$, then they are dependent conditional on any subset of $\mathbf{V}\setminus\left\{ X,Y\right\} $.
\end{defn}
Let $U$, $V$ be two distinct relational variables of the same perspective
$B$. We limit $U$ and $V$ to be non-intersectable to each other.
Otherwise, they must not be adjacent to each other by the definition
of RCM since an edge between intersectable relational variables yields
a feedback in a ground graph. If there is an edge $U\rightarrow V$
in $\mathbf{AGG}_{\mathcal{M}}$, 
\[
\forall_{\mathbf{W}\subseteq\mathbf{RV}_{B}\setminus\left\{ U,V\right\} }\left(U\not\perp V\mid\mathbf{W}\right)_{\mathcal{M}}
\]
We can construct a skeleton $\sigma\in\Sigma_{\mathcal{S}}$ where
its corresponding ground graph $GG_{\mathcal{M}\sigma}$ satisfies
that $U|_{b}$ and $V|_{b}$ are singletons and $U|_{b}\cup V|_{b}$
are disjoint to $\left(\mathbf{RV}_{B}\setminus\left\{ U,V\right\} \right)|_{b}$
for $b\in\sigma\left(B\right)$. Lemma 4.4.1 by \citet{maier2014thesis}
describes a method to construct a \emph{minimal} skeleton to represent
$U$ and $V$ with a single $b\in\sigma\left(B\right)$. It guarantees
that $U|_{b}$ and $V|_{b}$ are singletons and every relational variable
$W\in\mathbf{RV}_{B}\setminus\left\{ U,V\right\} $ satisfies $W|_{b}\cap U|_{b}=\emptyset$
and $W|_{b}\cap V|_{b}=\emptyset$.

\paragraph{Orientation-Faithfulness}
\begin{defn}[Orientation-Faithfulness \citep{Ramsey2006}]
Given a set of variables $\mathbf{V}$ whose causal structure can
be represented by a DAG $G$, let $\left\langle X,Y,Z\right\rangle $
be any unshielded triple in $G$. 
\begin{itemize}[label=(O1)]
\item if $X\rightarrow Y\leftarrow Z$, then $X$ and $Z$ are dependent
given any subset of $\mathbf{V}\setminus\left\{ X,Z\right\} $ that
contains $Y$; 
\end{itemize}

\begin{itemize}[label=(O2)]
\item otherwise, $X$ and $Z$ are dependent conditional on any subset
of $\mathbf{V}\setminus\left\{ X,Z\right\} $ that does not contain
$Y$.
\end{itemize}
\end{defn}
Let $U$, $V$, and $W$ be three distinct relational variables of
the same perspective $B$ forming an unshielded triple in $\mathbf{AGG}_{\mathcal{M}}$.
Similarly, $V$ is not intersectable to both $U$ and $W$. The condition
(O1) can be written as 
\[
\forall_{\mathbf{T}\subseteq\mathbf{RV}_{B}\setminus\left\{ U,W\right\} }(U\!\not\perp\!W\mid\mathbf{T}\cup\{V\})_{\mathcal{M}}
\]
if edges are oriented as $U\rightarrow V\leftarrow W$ in $\mathbf{AGG}_{\mathcal{M}}$.
Otherwise, 
\[
\forall_{\mathbf{T}\subseteq\mathbf{RV}_{B}\setminus\left\{ U,W\right\} }(U\!\not\perp\!W\mid\mathbf{T}\setminus\{V\})_{\mathcal{M}}
\]
for the condition (O2). Again, constructing a minimal skeleton for
$U$, $V$, and $W$ guarantees that no $T\in\mathbf{RV}_{B}\setminus\left\{ U,V,W\right\} $
can represent any item in $\left\{ U|_{b},V|_{b},W|_{b}\right\} $.
Thus, the existence of $V$ in the conditional determines (in)dependence
in the ground graph induced from the minimal skeleton.

\paragraph{Learning RCM with Non-complete AGG}

RCD (Relational Causal Discovery, \citet{maier2013rcd}) is an algorithm
for learning the structure of an RCM from relational data. In learning
RCM, AGG plays a key role: AGG is constructed using CI tests to obtain
the relational dependencies of an RCM. The lack of completeness of
AGG for relational \emph{d}-separation in RCM raises questions about
the correctness of RCD. A careful examination of AGG through the lens
of faithfulness suggests that \emph{adjacency-faithful} and \emph{orientation-faithful}
conditions can be applied to $\mathbf{AGG}_{\mathcal{M}}$ to recover
correct partially-oriented dependencies for an RCM. However, it is
still unclear whether RCD recovers maximally-oriented dependencies
with the acyclicity of AGG (i.e., relational variables) not the acyclicity
of RCM (i.e., attribute classes). This raises the possibility of an
algorithm for learning the structure of an RCM from relational data
that does not require the intermediate step of constructing a lifted
representation.

\section{CONCLUDING REMARKS}

There is a growing interest in relational causal models \citep{maier2010rpc,maier2013rcd,maier2013rds-workshop,maier2014thesis,arbour2014psm,marazopoulou2015trcm}.
A lifted representation, called \emph{abstract ground graph} (AGG),
plays a central role in reasoning with and learning of RCM. The correctness
of the algorithm proposed by \citet{maier2013rcd} for learning RCM
from data relies on the \emph{soundness and completeness} of AGG for
\emph{relational d-separation} to reduce the learning of an RCM to
learning of an AGG. We showed that AGG, as defined in \citep{maier2013rcd},
does \emph{not} correctly abstract all ground graphs. We revised the
definition of AGG to ensure that it correctly abstracts all ground
graphs. We further showed that AGG representation is \emph{not complete}
for relational \emph{d}-separation\emph{,} that is, there can exist
conditional independence relations in an RCM that are not entailed
by AGG. Our examination of the relationship between the lack of completeness
of AGG for relational \emph{d}-separation and \emph{faithfulness}
suggests that weaker notions of completeness, namely \emph{adjacency
faithfulness }and \emph{orientation faithfulness} between an RCM and
its AGG can be used to learn an RCM from data. Work in progress is
aimed at: 1) identifying the necessary and sufficient criteria for
guaranteeing the completeness of AGG for relational \emph{d}-separation;
2) establishing whether the RCD algorithm outputs a maximally-oriented
RCM even when the completeness of AGG for relational \emph{d}-separation
does not hold; and 3) devising a structure learning algorithm that
does not rely on a lifted representation.

\section*{APPENDIX}

We first prove \prettyref{lem:intersectability-nec-suf} in \prettyref{sub:Intersectability-and-IV}.
\begin{lem*}
Given a relational schema $\mathcal{S}$, let $P$ and $Q$ be two
different relational paths satisfying the (necessary) criteria of
\citet{maier2014thesis} and $\left|Q\right|\leq\left|P\right|$.
Let $m$ and $n$ be $\mathsf{LLRSP}(P,Q)$ and $\mathsf{LLRSP}(\tilde{P},\tilde{Q})$,
respectively. Then, $P$ and $Q$ are intersectable if and only if
$m+n\leq\left|Q\right|$.\end{lem*}
\begin{proof}
(If part) If $m+n\leq\left|Q\right|$, then we can construct a skeleton
$\sigma$ such that $P|_{b}\cap Q|_{b}\neq\emptyset$ for some $b\in\sigma(P_{1})$
by adding unique items for $Q$ and for $P^{m+1:\left|P\right|-n}$
and complete the skeleton in the same manner as shown in Lemma 3.4.1
\citep{maier2014thesis}. Note that if $m+n=\left|Q\right|$, then
$\left|P\right|\geq\left|Q\right|+2$ since $P\neq Q$ and a relational
path is an alternating sequence. This guarantees that there are at
least two items for $P^{m+1:\left|P\right|-n}$.

(Only if part) Let $c$ be in $P|_{b}\cap Q|_{b}$ for some arbitrary
skeleton $\sigma\in\Sigma_{\mathcal{S}}$ and $b\in\sigma(P_{1})$.
Then, there should be two lists of items corresponding to $P$ and
$Q$ sharing the first $m$ and the last $n$. The condition $m+n>\left|Q\right|$
implies $Q|_{b}$ is a singleton set. We define
\[
\mathbf{p}=\langle p_{1},\dotsc,p_{m},p_{\left|P\right|-n+1},\dotsc,p_{\left|P\right|}\rangle
\]
and
\[
\mathbf{q}=\langle q_{1},\dotsc,q_{\left|Q\right|}\rangle,
\]
where $\{q_{\ell}\}=Q^{1:\ell}|_{b}$ and $\{p_{\ell}\}=P^{1:\ell}|_{b}$
for $1\leq\ell\leq m$, and $p_{\left|P\right|-l+1}\in\tilde{P}^{1:l}|_{c}$
for $1\leq l\leq n$. We can see that $p_{1}=q_{1}=b$ and $p_{\left|P\right|}=q_{\left|Q\right|}=c$.
Moreover, 
\[
p_{m}=q_{m}=q_{\left|Q\right|-(\left|Q\right|-m)}=p_{\left|P\right|-\left(\left|Q\right|-m\right)}
\]
by the definition of $\mathsf{LLRSP}$. If $\left|Q\right|<\left|P\right|$,
then $m\neq\left|P\right|-\left|Q\right|+m$ and $m$th item for $P$
is repeated at $\left|P\right|-(\left|Q\right|-m)$th, which violates
the BBS. Otherwise, it is not the case, since $\left|P\right|=\left|Q\right|$
implies $\mathbf{p}=\mathbf{q}$ and, hence, $P=Q$ by the definition
of $\mathsf{LLRSP}$, which contradicts the assumption that $P$ and
$Q$ are different relational paths.
\end{proof}
We provide proofs for two claims regarding the counterexample in \prettyref{sub:COUNTEREXAMPLE}.

\medskip{}

\begin{claim*}
$\left(\overline{P.X}\not\Perp\overline{S^{\prime}.Z}\mid\overline{Q.Y}\right)_{\mathbf{AGG}_{\mathcal{M}}}$.\end{claim*}
\begin{proof}
By the definition of RVE, there are RVEs $P.X\rightarrow Q.Y$ and
$Q.Y\leftarrow S.Z$ in $\mathbf{AGG}_{\mathcal{M}}$ since $P=Q\Join_{6}D_{1}$
and $S\in Q\Join_{4}D_{2}$. Moreover, there is an IVE $Q.Y\leftarrow S.Z\cap S^{\prime}.Z$
in $\mathbf{AGG}_{\mathcal{M}}$ since 1) $S$ and $S^{\prime}$ are
\emph{intersectable, }2) there is an RVE $Q.Y\leftarrow S.Z$, and
3) $\left\langle Q,D_{2},S,S^{\prime}\right\rangle $ is \emph{co-intersectable}
(see \prettyref{fig:co-intersectability}).\footnote{Note that the original definition of $\mathbf{AGG}_{\mathcal{M}}$
does not check \emph{co-intersectability} and $Q.Y\leftarrow S.Z\cap S^{\prime}.Z$
is granted.} Since $P.X\rightarrow Q.Y\leftarrow S.Z\cap S^{\prime}.Z$ and $S.Z\cap S^{\prime}.Z\in\overline{S^{\prime}.Z}$,
we derive $\left(P.X\not\Perp\overline{S^{\prime}.Z}\mid Q.Y\right)_{\mathbf{AGG}_{\mathcal{M}}}$,
which implies $\left(\overline{P.X}\not\Perp\overline{S^{\prime}.Z}\mid Q.Y\right)_{\mathbf{AGG}_{\mathcal{M}}}$.
Furthermore, conditioning on $\overline{Q.Y}$, compared to $Q.Y$,
does not block any possible \emph{d}-connection paths between $\overline{P.X}$
to $\overline{S^{\prime}.Z}$ since there are only incoming edges
to $\overline{Q.Y}$. Finally, $\left(\overline{P.X}\not\Perp\overline{S^{\prime}.Z}\mid\overline{Q.Y}\right)_{\mathbf{AGG}_{\mathcal{M}}}$
holds.\end{proof}
\begin{claim*}
There is no $\sigma\in\Sigma_{\mathcal{S}}$ and $b\in\sigma\left(B\right)$
such that 
\[
\left(P.X|_{b}\not\Perp S^{\prime}.Z|_{b}\mid Q.Y|_{b}\right)_{GG_{\mathcal{M}\sigma}}.
\]
\end{claim*}
\begin{proof}
Suppose that there exist such a skeleton $\sigma$ and base $b\in\sigma\left(B\right)$
satisfying $\left(P.X|_{b}\not\Perp S^{\prime}.Z|_{b}\mid Q.Y|_{b}\right)_{GG_{\mathcal{M}\sigma}}$.
Every terminal set for $P$, $Q$, and $S^{\prime}$ given the base
must not be empty because of the definition of \emph{d}-separation
and the fact that attribute classes $X$ and $Z$ are connected only
through $Y$ (i.e., $Y$ is a collider). Since every cardinality is
$\mathsf{one}$, terminal sets must be singletons. Let $\left\{ i_{x}\right\} =P.X|_{b}$,
$\left\{ i_{y}\right\} =Q.Y|_{b}$, and $\left\{ i_{z}\right\} =S^{\prime}.Y|_{b}$.
Furthermore, since $i_{x}$ and $i_{z}$ must be \emph{d}-connected
given $i_{y}$, $GG_{\mathcal{M}\sigma}$ must have two edges $i_{x}\rightarrow i_{y}\leftarrow i_{z}$,
which requires $i_{x}\in D_{1}|_{i_{y}}$ and $i_{z}\in D_{2}|_{i_{y}}$.
However, due to BBS and cardinality constraints (i.e., $\mathsf{one}$),
there exists only one possible structure (see \prettyref{fig:counterexample-for-completeness})
where $i_{x}$ and $i_{z}$ are the cause of $i_{y}$ while satisfying
all previously mentioned conditions except $\left\{ i_{z}\right\} =S^{\prime}.Y|_{b}$.
In other words, the constraint $\left\{ i_{z}\right\} =S^{\prime}.Y|_{b}$
violates with the set of the rest of conditions. Hence, there exists
no such skeleton and base.
\end{proof}

\subsubsection*{Acknowledgments}

This research was supported in part by the Edward Frymoyer Endowed
Professorship held by Vasant Honavar and in part by the Center for
Big Data Analytics and Discovery Informatics which is co-sponsored
by the Institute for Cyberscience, the Huck Institutes of the Life
Sciences, and the Social Science Research Institute at the Pennsylvania
State University.

\bibliographystyle{apalike}
\bibliography{rcm_2hm}

\end{document}